\documentclass{article}

\usepackage[preprint]{neurips_2025}

\usepackage{neurips_2025} %[preprint]

\usepackage[commands, environments, theorems]{AVT}

\usepackage{graphicx}
\usepackage{subcaption}
\usepackage{algorithm}
\usepackage{algpseudocode}
\usepackage{bbm}
\usepackage{calc} 
\usepackage{bbm}
\usepackage[export]{adjustbox}
\newsavebox{\mygridbox}

\DeclareMathOperator{\KL}{\mathsf{KL}}

\DeclareMathOperator{\rP}{\mathbb{P}}
\DeclareMathOperator{\Id}{\mathrm{Id}}

\def\abs#1{\left|#1\right|}
\def\norm#1{\left\lVert#1\right\rVert}
\def\curly#1{\left\{#1\right\}}
\def\square#1{\left[#1\right]}

\def\v#1{\boldsymbol{#1}}

\hypersetup{colorlinks,citecolor=blue}

\newcommand{\expectation}{\mathbb{E}}

\newcommand{\cD}{\mathcal{D}}

\newcommand{\cN}{\mathcal{N}}

\newcommand{\ind}{\mathbbm{1}}
\newcommand{\innerproduct}[2]{\left\langle #1, #2 \right\rangle}

\DeclarePairedDelimiterX{\infdiv}[2]{(}{)}{%
  #1\;\delimsize\|\;#2%
}

\usepackage[utf8]{inputenc} % allow utf-8 input
\usepackage[T1]{fontenc}    % use 8-bit T1 fonts
\usepackage{hyperref}       % hyperlinks
\usepackage{url}            % simple URL typesetting
\usepackage{booktabs}       % professional-quality tables
\usepackage{amsfonts}       % blackboard math symbols
\usepackage{nicefrac}       % compact symbols for 1/2, etc.
\usepackage{microtype}      % microtypography
\usepackage{xcolor}         % colors
\usepackage{todonotes}

\title{Adaptive Diffusion Guidance via Stochastic Optimal Control}

\author{%
   Iskander Azangulov\\
  Department of Statistics\\
  University of Oxford\\
  \texttt{iskander.azangulov@stats.ox.ac.uk} \\
  \And
  Peter Potaptchik \\
  Department of Statistics\\
  University of Oxford\\
  \texttt{peter.potaptchik@stats.ox.ac.uk} \\
  \And
  Qinyu Li\\
  Department of Statistics\\
  University of Oxford\\
  \texttt{qinyu.li@stats.ox.ac.uk} \\
  \And
  Eddie Aamari \\
  Département de mathématiques et applications \\
  École normale supérieure, Université PSL, CNRS \\
  \texttt{eddie.aamari@ens.fr} \\
  \And
  George Deligiannidis \\
  Department of Statistics\\
  University of Oxford\\
  \texttt{george.deligiannidis@stats.ox.ac.uk} \\
  \And
  Judith Rousseau \\
  CEREMADE, Universit\'e Paris-Dauphine,\\
   PSL University,  CNRS\\
  \texttt{rousseau@ceremade.dauphine.fr} \\
}

\begin{document}

\maketitle

\begin{abstract}
  Guidance is a cornerstone of modern diffusion models, playing a pivotal role in conditional generation and enhancing the quality of unconditional samples. However, current approaches to guidance scheduling—determining the appropriate guidance weight—are largely heuristic and lack a solid theoretical foundation. This work addresses these limitations on two fronts. First, we provide a theoretical formalization that precisely characterizes the relationship between guidance strength and classifier confidence. Second, building on this insight, we introduce a stochastic optimal control framework that casts guidance scheduling as an adaptive optimization problem. In this formulation, guidance strength is not fixed but dynamically selected based on time, the current sample, and the conditioning class, either independently or in combination. By solving the resulting control problem, we establish a principled foundation for more effective guidance in diffusion models.
\end{abstract}
\section{Introduction}
Diffusion models \citep{sohl2015deep, ho2020denoising, song2021scorebased} have emerged as a dominant paradigm in generative modeling, achieving state-of-the-art results across various domains, including text-to-image generation \citep{rombach2022high, ramesh2022hierarchical, Zhang_2023_ICCV}, video synthesis \citep{ho2022video, singer2022make, blattmann2023stable}, and molecular design~\citep{bao2022equivariant, weiss2023guided, schneuing2024structure}. These models operate by defining a forward diffusion process $X_t$ that gradually adds noise to data, $X_0 \sim p$, and then generating samples by learning to reverse this noising process. The backward process is driven by the score function $\nabla \log p_t$, where $p_t$ is the density of $X_t$. In practice, the score is approximated by a neural network trained via score matching \citep{song2021scorebased}.

While incredibly powerful, the raw output from these diffusion models often benefits from refinement to enhance sample quality and ensure alignment with specific desired conditions or characteristics. Such tasks are widespread, from generating images based on text to steering protein generation to meet specific constraints. \emph{Guidance} \citep{dhariwal2021diffusion, ho2022classifier} techniques have, therefore, become an indispensable part of generative modeling, ubiquitously employed across a variety of tasks \citep{ramesh2022hierarchical, saharia2022photorealistic, schiff2025simpleguidancemechanismsdiscrete}. 

Among guidance techniques, Classifier-Free Guidance (CFG) \citep{ho2022classifier} has become a cornerstone, significantly improving conditional generation.
%and the quality of unconditional samples. 
Let $c$ be a conditioning input such as a class label or a textual prompt.
CFG involves training a single diffusion model on complementary objectives to simultaneously learn the conditional scores $\grad \log p(x|c)$ for different conditions $c$ and the unconditional score $\grad \log p(x)$. During sampling, CFG modifies the standard reverse diffusion process by introducing a guidance strength parameter $w \in \R$. This parameter induces a guided score $\nabla \log p_t(x|c) + w(\nabla \log p_t(x|c) - \nabla \log p_t(x))$, which interpolates between the conditional $(w=0)$ and the unconditional $(w=-1)$ scores. 

In practice, the guidance strength $w$ is typically chosen to be much greater than 0 to amplify the influence of the conditioning information. For example, in image generation, $w$ commonly ranges from 3 to 16 \citep{ho2022classifier, podell2023sdxl, Li_2024_CVPR, stable-diffusion-manual}, though this is just a heuristic choice and the theoretical understanding of the guiding mechanism is still limited. We discuss this in more detail in the next section.

\subsection{Related Works}\label{sec:related_works}

While choosing constant guidance is standard practice, recent works have explored various heuristics for improving guidance via non-constant weight schedules. For instance, \citet{chang2023muse, gao2023mdtv2} proposed schedules where guidance strength increases with the current noise level. \citet{kynkaanniemi2024applying} proposed only applying guidance at empirically specified time intervals. Additionally, \citet{shen2024rethinkingspatialinconsistencyclassifierfree} proposed using different guidance weights across different semantic regions, arguing that a fixed guidance weight results in spatial inconsistencies.

From a theoretical point of view, the understanding of the guidance mechanism is still limited. A common intuition is that guidance with weight $w$ effectively samples from a "tilted" target distribution proportional to $p(x)p(c|x)^w$, however, as shown in \citet{chidambaram2024does, bradley2024classifier}, this is not true. Another point of ambiguity, and a widely held misconception, is whether guided sampling ($w \neq 0$) ensures that samples remain within the true conditional data support, with some believing that guidance may prevent the recovery of this support. This concern about losing the data support is, interestingly, contradictory to the aforementioned "tilted" distribution intuition. 

Recently, \citet{skreta2025feynman} proposed a Feyman-Kac based corrector that runs a Sequential Monte Carlo scheme along guided trajectories to sample exactly from $\propto p(x)p(c|x)^w$. Also, \citet{domingoenrich2025adjointmatchingfinetuningflow} suggested a Stochastic Optimal Control based method to train a guiding vector, so that the resulting samples are from $\propto p(x)\exp(r(x))$ for an arbitrary terminal reward $r(x)$.

\subsection{Our Contributions}

This work addresses these limitations on two main fronts. First, we address a critical gap in the current understanding of guidance: the ambiguity of the precise distribution being sampled, and a lack of formal guarantees when $w \neq 0$. 
Our work establishes that (i) generated samples are guaranteed to remain within the conditional data manifold; and (ii) applying guidance with any positive strength $w>0$ increases $p(c|x)$ both with high probability and on average.

Second, building on these theoretical insights, we introduce a novel framework that recasts guidance scheduling as a Stochastic Optimal Control (SOC) problem.  In this framework, the guidance strength is no longer a fixed parameter, but is dynamically adjusted based on time, the current sample state, and the conditioning class $c$. By formulating the resulting control problem, we establish a principled foundation for effective adaptive guidance. The resulting framework optimizes a scalar-valued function $w=w_t(x,c)$ and only requires estimates of $\grad \log p_t(x)$ and $\grad \log p_t(x|c)$, making it applicable in the CFG setting.

Finally, we propose and implement a scalable numerical method based on the adjoint method to efficiently solve this control problem. The open-source code for our framework is publicly available at \href{https://github.com/imbirik/adaptive-guidance}{github.com/imbirik/adaptive-guidance}.

\paragraph{Concurrent work.} While we were finalizing our paper, a concurrent work~\citep{li2025provableefficiencyguidancediffusion} appeared, presenting a result similar to our~\cref{cor:guarantees}. Using a different approach, they also show that applying guidance $w>0$ reduces $\E p^{-1}(c|Y^w_t)$ compared to the unguided case $w\equiv 0$.

\section{Background}
\subsection{Diffusion Models via Stochastic Differential Equations}
Diffusion models define a generative process by reversing a forward stochastic differential equation (SDE) that progressively corrupts data into noise. Let $X_0 \sim p_0 = p$ denote the data distribution in $\R^D$. The forward (noising) process is governed by the Ornstein–Uhlenbeck SDE:
\begin{equation}
\label{eq:forward}
    dX_t = - X_tdt + \sqrt{2}dB_t, \quad t \in [0, T],
\end{equation}
where $B_t$ is a standard Wiener process. We denote the density of $X_t$ by $p_t$ and remark that
\begin{equation}
    X_t |X_0 := \mathcal{N}\left( e^{-t}X_0, (1-e^{-2t}) \Id_D \right).
\end{equation}
Under mild assumptions~\citep{ANDERSON1982313}, the reverse-time process $\curly{Y_t}_{t\in [0,T]}:=\curly{X_{T-t}}_{t\in [0,T]}$ satisfies the SDE
\begin{equation}
    dY_t = \left[ Y_t + 2\nabla\log p_{T-t}(Y_t) \right]dt + \sqrt{2}d\bar{B}_t,
\end{equation}
where $\bar{B}_t$ is a standard Wiener process in the reverse-time direction.
In practice, $\nabla \log p_t$ is approximated by a neural network trained via score matching and sampling $Y_T \sim p_T$ is typically approximated by $Y_T \sim \mathcal{N}(0,I_D)$. We refer to~\cite{karras2022elucidatingdesignspacediffusionbased} for further details.

\subsection{Guidance}
\label{sec:background_cfg}
Conditional generation from $p(x|c)$ can be achieved by replacing $\nabla \log p_t(x)$ with the conditional score $\nabla \log p_t(x|c)$. However, in order to increase adherence to the condition $c$, strong guidance is often employed by introducing a guidance strength $w > 0$ and instead using the guided score  
\begin{equation}
\label{eq:cfg_score}
\nabla\log p_t(x|c) + w\nabla \log p_t(c|x).
\end{equation}
So, the backward dynamics is given by
\begin{multline}\label{eqn:cfg_full}
dY_t^w = \Big[Y^w_t + 2\grad \log p_{T-t}(Y_t^w|c) + 2w\nabla \log p_t(c|Y_t^w) \Big]dt + \sqrt{2}dB_t.
\end{multline} 
Guided scores can be obtained either by using learned classifiers $p_t(c|x)$ for noisy data, or by applying Bayes' rule $\nabla \log p_t(c|x) = \nabla \log p_t(x|c) - \nabla \log p_t(x)$, yielding  Classifier-Free Guidance (CFG) which directly uses a score network for both the unconditional $\nabla \log p_t(x)$ and the conditional $\nabla \log p_t(x|c)$ scores, thereby avoiding a separate classifier.

\subsection{Stochastic Optimal Control}
Stochastic optimal control deals with the problem of controlling a system whose state evolves over time according to a stochastic process, with the goal of optimizing a certain objective. Consider a stochastic process $Y_t^w$ whose dynamics can be influenced by a control variable $w_t(Y_t^w)$. Define a reward function $R(Y^w, w)$ that quantifies the desirability of the terminal state $Y_T^w$ and the entire trajectory $(Y_t^w)_{t\in [0,T]}$. We consider the following objective for $w$
\[
    R(w) := \mathbb{E} \left[R(Y^w, w)\right ]=\mathbb{E} \left[ \Phi(Y_T^w) + \int_0^T r(t, Y_t^w, w_t(Y_t^w)) dt  \right],
\]
where $\Phi$ is a terminal reward function and $r$ is an instantaneous reward function.
%\todo{Are $\Phi$ and $r$ kept general? User-defined? Specific?} 
The goal is to find the control policy $w_t^*(Y_t^w)$ in some class $\mathbb{U}$ that maximizes this reward. The value function $V(t, x)$ is defined as the optimal expected reward obtained by starting from state $x$ at time $t$:
\begin{equation}
    V_t(x)\! = \!\sup_{w \in \mathbb{U}} \mathbb{E} \!\!\left[ \Phi(Y_T^w) \! + \!\!\int_t^T \!\!\!\!\!\!r(s, Y_s^w, w_s(Y_s^w)) ds \Big | Y_t^w = x \right]\!\!.
\end{equation}
Under suitable regularity conditions, the value function satisfies the Hamilton-Jacobi-Bellman (HJB) equation, a backward-in-time partial differential equation:
\begin{equation}
    -\frac{d}{d t} V_t(x) = \sup_{w\in \mathbb{U}} \left\{ \mathcal{A}^w V_t(x) + r(t, x, w) \right\},
\end{equation}
where $\mathcal{A}^w$ is the infinitesimal generator of the stochastic process $Y_t^w$, which depends on the control $w$. Solving the HJB equation yields the optimal value function $V^*_t(x)$ and the optimal control policy 
\begin{equation}
w_t^*(x) = \arg\max_{w\in \mathbb{U}} \{ \mathcal{A}^w V_t(x) + r(t, x, w) \}.
\end{equation}
 We refer to~\citep{fleming2006controlled} for further details.

\section{Theoretical Results for Guidance}
\label{sec:theory}

In this section, we present our two main theoretical results for guidance. The first result establishes a connection between the guidance strength and $\log p_t(c|Y^w_{T-t})$ -- a measure of the alignment of the guided samples with the classifier at time $t$. Our second result shows that the application of guidance preserves the support of the true conditional distribution, $\supp p(x|c)$. We defer proofs to Appendix~\ref{app:proofs}.

In the remainder of this paper, we consider a generalization of the guidance introduced in \cref{sec:background_cfg}, by allowing $w$ to depend on time, position, and the conditioning class. To simplify notation, we suppress these dependencies and write $w := w_t(Y_t^w, c)$.

\subsection{Analysis of Classifier Confidence Along Guided Trajectories}
\label{sec:theory_1}
Let us introduce the backward running "guiding field" $G_t$:
\begin{align}
     \label{eqn:G_t}
    G_t(x) := \log p_{T-t}(c|x) - \log p(c) =  \log p_{T-t}(x|c) - \log{p_{T-t}(x)}, \nonumber
\end{align}
where we also suppress the dependence on $c$ in the notation for $G_t$. So, \eqref{eqn:cfg_full} reads as
\£
\label{eqn:cfg_full}
dY_t^w = \Big[Y^w_t + 2\grad \log p_{T-t}(Y_t^w|c) + 2w\grad G_t(Y_t^w) \Big]dt + \sqrt{2}dB_t.
\£

Combining Ito's lemma and the Fokker-Planck equation, we obtain our main technical lemma. We defer the proof to the Appendix.
\begin{restatable}{lemma}{ItoForGt}
\label{lemma:ito}
    \begin{equation} \label{eqn:ito}
        d G_t(Y^w_t) = (1+2w)\norm{\grad G_t(Y_t^w)}^2dt + \sqrt{2}\grad G_t(Y_t^w) \cdot dB_t.
    \end{equation}
\end{restatable}
Taking expectations in (\ref{eqn:ito}) and using (\ref{eqn:G_t}), we can derive an expression for the mean likelihood of target class $c$ at time $t$, $\E \log p_{T-t}(c|Y_t^w)$. It is given by an integral of the \textit{known} term $\grad G_t$, along the paths of the guided process $Y^w$:
\begin{align} \label{eqn:decomp}
    \E \log p_{T-t}(c|Y_t^w) =C + \E\int_0^{t} (1+2w)\norm{\grad G_s(Y_s^w)}^2ds,
\end{align}
where the constant $C=\log p(c) + \E \log \frac{p_T(Y_0^w|c)}{p_T(Y^w_0)}$ does not depend on the choice of the guidance function $w$. Moreover, we note that the quadratic variation of $G_t$ satisfies $d\langle G \rangle_t = 2\norm{\grad G_t}^2dt$. Therefore, the Dol\'eans-Dade stochastic exponential 
%solving $dS_t = S_t\grad G_t\cdot dB_t$ is
of the martingale $-\sqrt{2}\nabla G_s(Y_s^w) \cdot dB_s$ is equal to
\begin{align}    
\label{eq:S_t}
    S_t := \exp\del{-G_t + \int_0^t 2w_s\norm{\grad G_s}^2ds} =\frac{p(c)}{p(c|Y^w_t)}\exp\del{\int_0^t 2w_s\norm{\grad G_s}^2ds}. %\frac{p_{T-t}(Y^w_t)}{p_{T-t}(Y^w_t|c)}.
\end{align}

$S_t$ is therefore a non-negative martingale on $[0,T)$ and a super-martingale on $[0,T]$. We refer the reader to~\cite{karatzas1998brownian} for more details on the required martingale theory. Leveraging the super-martingale property, we obtain the following corollary.

\begin{restatable}{corollary}{CorollaryMartingale} By Doob's inequality for any $\delta > 0$ with probability at least $1-\delta$ sampled guided trajectory $Y^w$ satisfies
\label{cor:guarantees}
    \£
    \label{eq:doob}
    \frac{p_{T-t}(Y^w_t|c)}{p_{T-t}(Y^w_t)} \ge 
    \delta\exp\square{\int_0^t 2w_s\norm{\grad G_s(Y^w_s)}^2ds}\Big)
    \£
    for all $t\in [0,T]$.
    % and if we take $\log$, with probability at least $1-\delta$
    % \[
    % \log p (c|Y_T^w) \ge \log p(c) + \log \delta + \int_0^T 2w_t\norm{\grad G_t}^2dt.
    % \]
    Moreover, if we condition on the event $Y^w_t=x$, then for any $\gamma \in [0, T-t)$
    \£
    \label{eq:bar_S_t_martingale}
       p^{-1}_{T-t}(c|x) =
       \E\square{p^{-1}_{T-t-\gamma}(c|Y^w_{t+\gamma})\exp\del{\int_t^{t+\gamma} \!\!\!\!\!\!\!2w\norm{\grad G_s}^2ds}\Big| Y^w_t = x}.
    \£
\end{restatable}
Since the term in the integrand is non-negative, and equal to zero if $w\equiv 0$ we get that for any $x$ s.t. $\grad G_t(x) \neq 0$
\[
\E\square{p^{-1}_{T-t-\gamma}(c|Y^w_{t+\gamma})\Big| Y^w_t = x} 
<  p^{-1}_{T-t}(c|x) = \E\square{p^{-1}_{T-t-\gamma}(c|Y^0_{t+\gamma})\Big| Y^0_t = x}.
\]
This implies that the addition of non-negative guidance guarantees increased alignment with the class label, compared to simulation without guidance. Moreover, by~\eqref{eq:doob} the term
$
\int_0^t 2w_s\norm{\grad G_s(Y^w_s)}^2ds
$
plays the role of the high probablity lower bound for $\log p_{T-t}(c|Y^w_t) - \log p(c)$.

\subsection{Support Recovery}
\label{sec:support_recovery}

In this section, we prove that under mild assumptions the terminal state of the guided process, $Y^w_T$, lies within $\supp p_0(\cdot|c)$. Specifically, we only assume: (i) the support of the unconditional distribution is compact, i.e.\ $\supp p_0 \subseteq B(0,R)$ for some $R > 0$; and (ii) the guidance is uniformly bounded, i.e.\ $0 < w \le C_{\max}$. Notably, we allow $p_0(\cdot|c)$ to be singular, and our result holds even when the guidance is not constant; (iii) The conditional class has positive weight $p(c) > 0$.

Below we provide a sketch of the proof, while the full proof can be found in the Appendix~\ref{app:proofs}.
First, we note that under these conditions, for all $t\in[0,T)$, by Tweedie's formula \citep{robbins1956} 
\£
\label{eq:naive_bound_G_t}
\norm{\grad G_t(x)}^2 = 
\frac{e^{2(t-T)} \norm{\E[X_0|X_{T-t}=x, c]-\E[X_0|X_{T-t}=x]}}{(1-e^{2(t-T)})^2} 
\le \frac{4e^{2(t-T)}}{(1-e^{2(t-T)})^2}R^2.
\£

Since the true conditional backward process $Y^0_t$ differs from guided process $Y^w_t$ by the drift term $w_t \grad G_t$. Girsanov's theorem combined with the data-processing inequality bounds the $\KL$ divergence
\[
\KL(Y^w_{T-\varepsilon}\| Y^0_{T-\varepsilon}) \le \KL(Y^w_{[0,T-\varepsilon]}\| Y^0_{[0,T-\varepsilon]}) 
= \expectation \int_0^{T-\varepsilon} w^2_t\norm{\grad G_t(Y_t^w)}^2dt,
\]
for all $\varepsilon > 0$, since Novikov's condition holds due to~\eqref{eq:naive_bound_G_t}. 
The integrand coincides with \eqref{eqn:decomp} up to a scalar $0 < \frac{w^2_t}{1+2w_t} < C_{\max}$, so 
\[
\KL(Y^w_{T-\varepsilon}\| Y^0_{T-\varepsilon}) \le
 \expectation \int_0^{T-\varepsilon}\!\!\!\!\!\!\! \frac{w^2_t}{1+2w_t} (1+2w_t)\!\norm{\grad G_t(Y_t^w)}^2\!\!dt\!\le\!-C_{\max} \log p(c).
\]
Therefore, for any $\varepsilon > 0$,  the $\KL$ divergence between $Y^w_{T-\varepsilon}$ and $Y^0_{T-\varepsilon}$ is uniformly bounded. Then support recovery follows from Donsker and Varadhan's variational formula.   
\begin{restatable}{theorem}{SupportRecovery}
    \label{thm:DDPM_support_recovery}
    Assume that there are $R, C_{\max} < \infty$ such that almost surely $\supp p_0 \subseteq B(0,R)$ and $w <C_{\max}$.
    Then almost surely $Y^w_T \in \supp p(\cdot |c)$. 
\end{restatable}
\begin{remark}
\label{rmk:full_supp_recovery}
    \cref{thm:DDPM_support_recovery} still holds for negative guidance as long as $w > -\frac{1}{2} + \varepsilon$ uniformly in $w$ for some $\varepsilon > 0$.
\end{remark}

\section{Adaptive Guidance Learning}
\label{sec:hjb}

As shown in the previous section, applying guidance encourages the growth of $ \E\log p(c|Y_T^w)$. However, by itself, the maximization of $ \E\log p(c|Y_T^w)$ might be an ill-defined problem. For example, in a simple Gaussian mixture case, this objective is maximized at infinity. Moreover, too much guidance can push trajectories to low-probability regions of $p_t(\cdot|c)$, where the score function is poorly learned. 

Both of these problems are addressed if we additionally assume that the guided trajectory distribution $Y^w$ and the unguided one $Y$ are close in $\KL$ divergence. Therefore, we introduce our SOC objective that balances these goals
\begin{equation} \label{eqn:objective}
    R(w) := \E \log p(c|Y_T^w) - \alpha\KL\infdiv{Y^w_{[0,T]}}{Y_{[0,T]}},
\end{equation}
where $\alpha > 0$ governs the tradeoff.
This leads to the following SOC problem: 
\begin{equation} \label{eqn:soc}
    w^* = \argmax_{w\in \mathbb{U}} \E \log p(c|Y_T^w) - \alpha\KL\infdiv{Y^w_{[0,T]}}{Y_{[0,T]}},
\end{equation}
where $\mathbb{U}$ is some class of guidance strength functions. In the most general case, which is the focus of this section, $\mathbb{U}$ is the space of adapted scalar-valued processes. By (\ref{eqn:decomp}) and Girsanov's theorem,
\begin{equation}\label{eqn:objective}
    R(w) = \E\int_0^T (1+2w_t-\alpha w_t^2)\norm{\grad G_t(Y_t^w)}^2dt + C,
\end{equation}
where $C$ is a constant that does not depend on the choice of guidance strength $w$. We introduce the value function 
\[
    V_t(x) =  
    \sup_w \E\square{\int_t^T (1+2w_t-\alpha w_t^2)\norm{\grad G_s(Y^w_s)}^2ds|Y_t^w=x}
\]
The associated HJB equation is given by
\begin{multline*}
    -\frac{d}{dt}V_t(x) =
\sup_{w \in \R}\Big[ \innerproduct{x+2\grad\log p_{T-t}(x|c)+2w\grad G_t(x)}{\grad V(t,x)} 
    + \Delta V(t,x)+(1+2w_t-\alpha w_t^2)\norm{\grad G_t}^2\Big],
\end{multline*}

with terminal condition $V_T\equiv 0$.

For fixed $t,x$, the objective in the supremum is a simple quadratic in $w_t(x)$. Thus 
\begin{equation}
    w^*_t(x) = \frac{\grad G_t\cdot \grad V_t + \norm{\grad G_t}^2}{\alpha \norm{\grad G_t}^2}
\end{equation}
and
\£
\label{eq:final_HJB}
 -\frac{d}{dt}V_t(x) =\frac{1}{\alpha}\curly{\frac{\grad G_t}{\norm{\grad G_t}}\cdot \grad V_t + \norm{\grad G_t}}^2 +\norm{\grad G_t}^2
 + \Delta V(t,x) + \innerproduct{x+2\grad\log p_{T-t}(x|c)}{\grad V(t,x)}.
\£

The resulting HJB equation can be estimated using finite difference methods or Physics-Informed Neural Networks (PINN) based objectives. Where the latter involves parameterization of the value function $V_t(x)$ as a neural network and its optimization by minimization of the squared difference between LHS and RHS of~\eqref{eq:final_HJB}.

However, both methods work only in low-dimensional cases, and do not allow for solutions that only depend on time $t$ or class $c$. In the next section, we discuss an alternative, scalable approach based on the adjoint state method. 

We emphasize that, unlike methods that learn a new guidance \textit{direction}, our method focuses on tuning only the \textit{scalar} guidance strength $w$, which allows for simpler network architectures and potentially less training data. 

% Requires in the preamble:
% \usepackage{graphicx}
% \usepackage[export]{adjustbox}

% Requires in the preamble:
% \usepackage{graphicx}
% \usepackage[export]{adjustbox}
% Requires in the preamble:
% \usepackage{graphicx}
% \usepackage[export]{adjustbox}

% Requires in the preamble:
% \usepackage{graphicx}
% \usepackage[export]{adjustbox}

\begin{figure*}[htbp]
    \vspace{-1.4cm}
    \centering
    % Left grid of 2 rows × 3 columns
    \begin{minipage}[t]{0.7\textwidth}
        \centering

        % Row 1
        \begin{subfigure}[t]{0.31\textwidth}
            \centering
            \includegraphics[width=\linewidth]{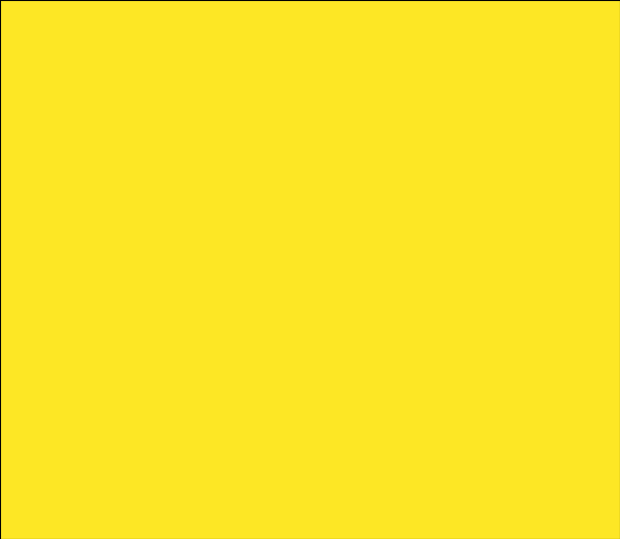}
            \caption{$T=0.0$}
            \label{fig:subfig1}
        \end{subfigure}
        \hfill
        \begin{subfigure}[t]{0.31\textwidth}
            \centering
            \includegraphics[width=\linewidth]{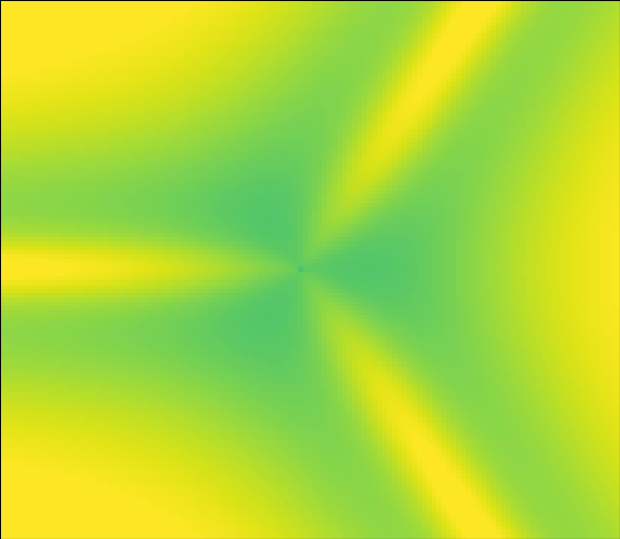}
            \caption{$T=0.1$}
            \label{fig:subfig2}
        \end{subfigure}
        \hfill
        \begin{subfigure}[t]{0.31\textwidth}
            \centering
            \includegraphics[width=\linewidth]{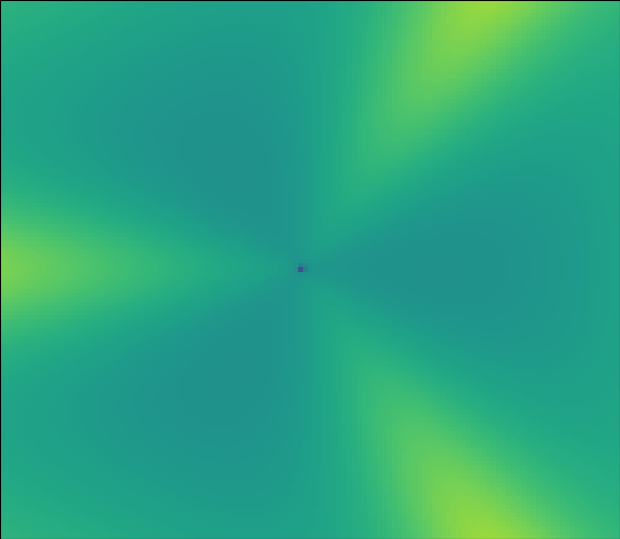}
            \caption{$T=0.5$}
            \label{fig:subfig3}
        \end{subfigure}

        \vspace{0.5cm}

        % Row 2
        \begin{subfigure}[t]{0.31\textwidth}
            \centering
            \includegraphics[width=\linewidth]{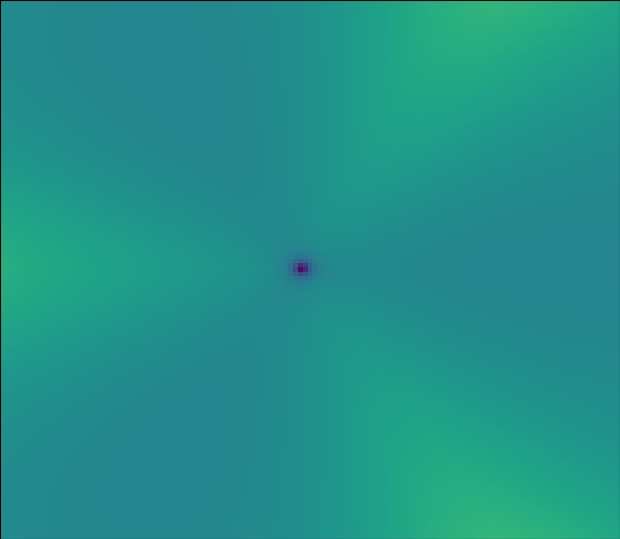}
            \caption{$T=1.0$}
            \label{fig:subfig4}
        \end{subfigure}
        \hfill
        \begin{subfigure}[t]{0.31\textwidth}
            \centering
            \includegraphics[width=\linewidth]{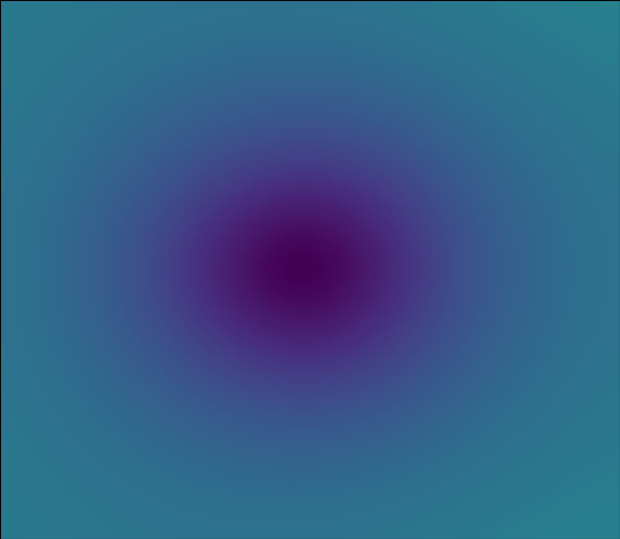}
            \caption{$T=2.5$}
            \label{fig:subfig5}
        \end{subfigure}
        \hfill
        \begin{subfigure}[t]{0.31\textwidth}
            \centering
            \includegraphics[width=\linewidth]{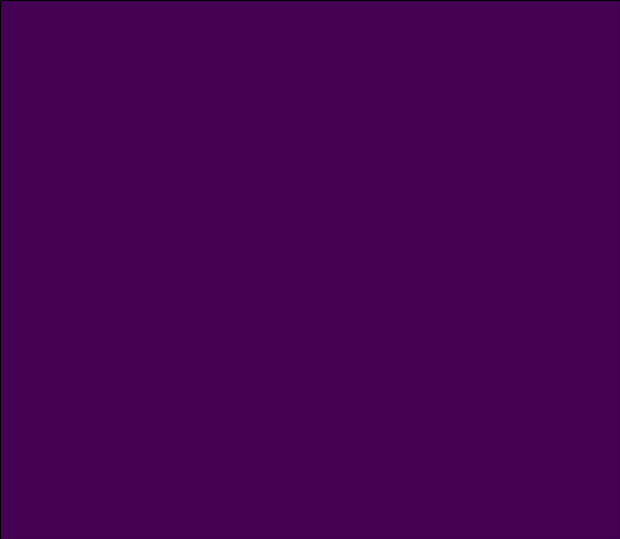}
            \caption{$T=5.0$}
            \label{fig:subfig6}
        \end{subfigure}
    \end{minipage}
    \hspace{0.5cm}
    % Right colorbar aligned to top
    \begin{minipage}{0.12\textwidth}
        %\centering
        \vspace*{1.4cm}
        \includegraphics[width=\textwidth, height=0.345\textheight]{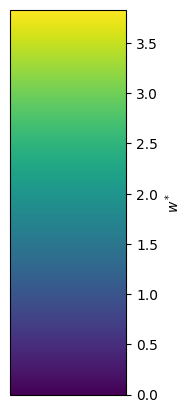}
    \end{minipage}

    \caption{Optimal guidance $w^*$ obtained by solving the HJB equation.}
    \label{fig:2x3_with_colorbar}
    \vspace{-0.5cm}
\end{figure*}

\section{Optimization of Stochastic Optimal Control Objective}
\label{sec:adjoint method}
The primary objective of this section is to solve the SOC problem defined in \eqref{eqn:soc} for general $\mathbb{U}$, i.e. find an optimal, adaptive guidance strength $w^*\in \mathbb{U}$ maximizing \eqref{eqn:objective}. We parameterize this guidance strength $w_{\theta}$ (potentially dependent on time, state, and conditioning prompt $c$) using a neural network whose parameters $\theta$ are optimized via stochastic gradient descent (SGD). 

To begin with we recall the exact formula~\eqref{eqn:objective} of our objective $R(w)$
\[
    R(w) = \E\int_0^T (1+2w_t-\alpha w_t^2)\norm{\grad G_t(Y_t^w)}^2dt + C,
\]
The main challenge is that the change of the guidance $w_{\theta}$ not only changes the instant reward $1+2w_t(\theta)-\alpha w^2_t(\theta))^2$ but also changes the distribution over trajectories $Y^w$. To capture this we introduce the \textit{sensitivity process}:
\[
Z_t(\tau) = \frac{\partial Y^w_t}{\partial w_\tau}
\]
capturing how the value of the process $Y^w_t$ at time $t > \tau$ changes w.r.t. to change of the process at initial time $\tau$. Applying the chain rule we obtain
\£
\label{eq: sensitivity of R}
    \frac{\partial R(w)}{\partial w_\tau}
    =2\E (1 - \alpha w_\tau)\norm{\grad G_t(Y_t^w)}^2 
    + \expectation\int_\tau^T{\color{blue}{2(1+2w_t-\alpha w_t^2)\grad^T G_t(Y_t^w)\grad^2 G_t(Y_t^w)}}Z_t(\tau) dt
\£

The general approach to computing the second integral is to introduce the adjoint process. First, using the chain rule, we write the SDE on the sensitivity process
\[
dZ_t(\tau)
=\Big[Z_t(\tau) + 2\grad^2 \log p_{T-t}(Y^w_t|c)Z_t(\tau) 
+ 2\delta(\tau-t)\grad G_\tau(Y_\tau^w)+ 2w_t\grad^2 G_t(Y_t^w)Z_t(\tau)\Big]dt
\]
implying that $Z_t(\tau)$ is driven by the ODE 
\[
\dot{Z}_t(\tau) =\Big[{\color{red}\Id + 2\grad^2 \log p_{T-t}(Y^w_t|c)}
{\color{red}+ 2w_t\grad^2 G_t(Y_t^w)}\Big]Z_t(\tau),
\]
with initial condition $Z_\tau = 2\grad G_\tau(Y_\tau^w).$
We consider an adjoint backward process $\lambda_t$ with initial condition $\lambda_T=0$ and satisfying the ODE 
\[    
\dot{\lambda}_t = -{\color{blue}\underbrace{2(1+2w_t-\alpha w_t^2)\grad^T G_t(Y_t^w)\grad^2 G_t(Y_t^w)}_{:=B_t}}
-\big[{\color{red} \underbrace{\Id + 2\grad^2 \log p_{T-t}(Y^w_t|c)+ 2w_t\grad^2 G_t(Y_t^w)}_{:=A_t}}\big]^T\lambda_t
\]
Then, integrating by parts, we see that 
\[
\langle \lambda_\tau, 2\grad G_\tau(Y_\tau^w) \rangle  =
\int_\tau^T {\color{blue}{2(1+2w_t-\alpha w_t^2)\grad^T G_t(Y_t^w)\grad^2 G_t(Y_t^w)}} Z_t dt
\]
that allows us to simplify \eqref{eq: sensitivity of R} and get
\[
\frac{\partial R(w)}{\partial w_\tau}=2\E (1 - \alpha w_\tau)\norm{\grad G_t(Y_t^w)}^2 + 2\langle \lambda_\tau, \grad G_\tau(Y_\tau^w) \rangle
\]

Finally, if the guidance $w = w_\theta$ is parameterized by a neural network with parameters $\theta$, the full gradient is obtained via the chain rule:
\[
\frac{dR}{d\theta} = \E \int_0^T \frac{\partial R(w)}{\partial w_\tau} \frac{\partial w_\tau}{\partial \theta} \, d\tau.
\]
In practice, we avoid the explicit and computationally expensive formation of $A_t, \grad^2 G_t \in \R^{D\times D}$ by leveraging Vector-Jacobian Products (VJPs). Pseudocode for this efficient implementation is provided in Appendix~\ref{sec:discrete-soc-pseudocode}.

\section{Experiments}

\begin{figure*}[htbp]
    \centering
    \includegraphics[width=0.41\textwidth,valign=t]{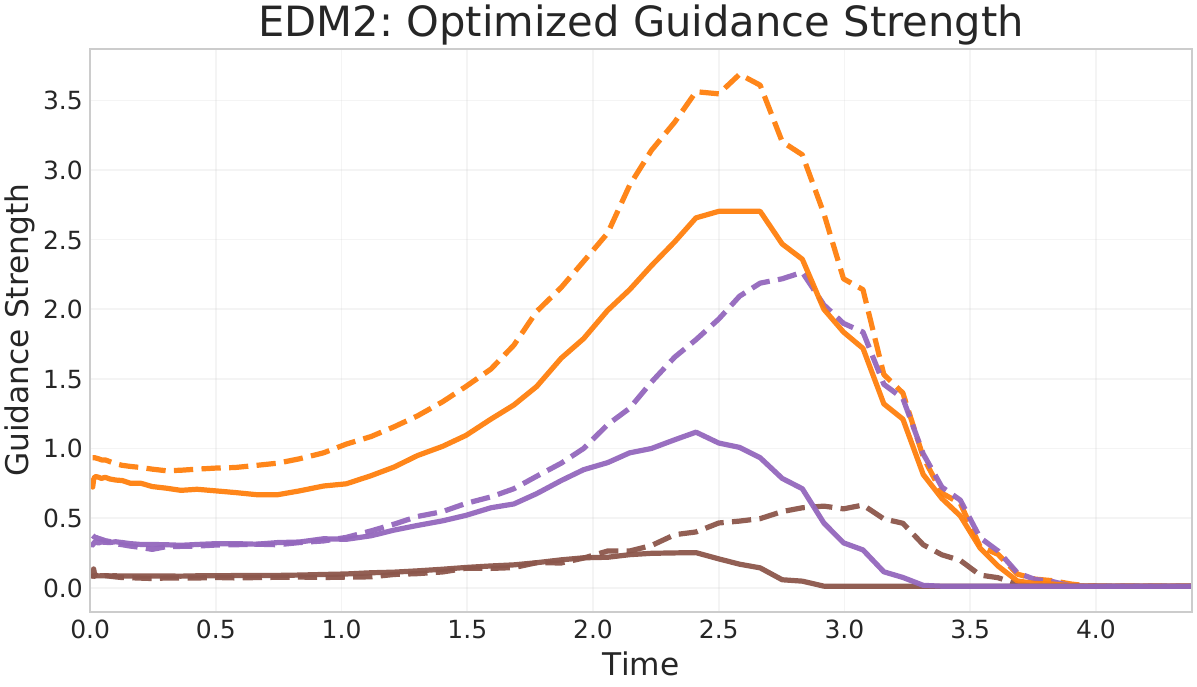}%
    \hspace{0.01\textwidth}%
    \includegraphics[width=0.41\textwidth,valign=t]{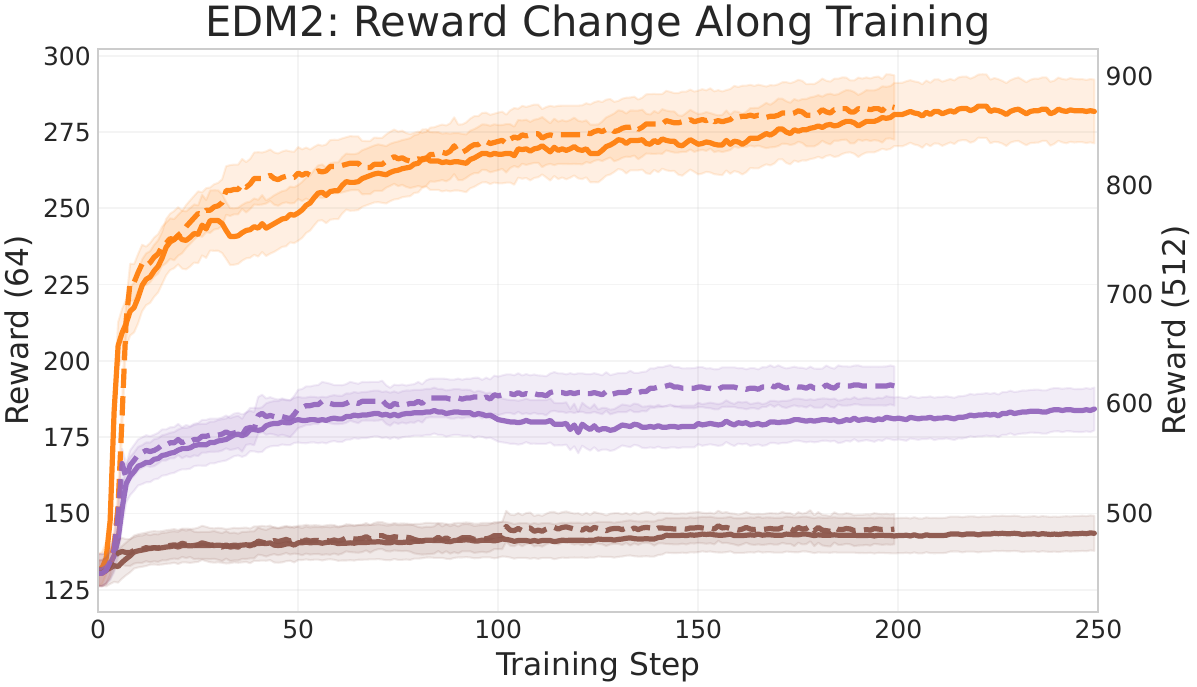}%
    \hspace{0.01\textwidth}%
    \raisebox{-0.3cm}{\includegraphics[width=0.15\textwidth,valign=t]{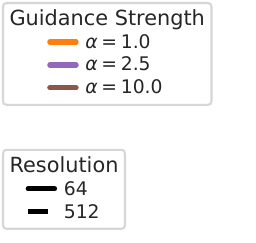}}%
    \caption{EDM2. Left: Guidance learned at the end of the training. Right: Change of the reward function $R(w_\theta)$ during the training with $\pm 1$ standard deviation.}
    \label{fig:edm2-plots}
\end{figure*}

\begin{figure*}[htbp]
    \centering
    \includegraphics[width=0.41\textwidth,valign=t]{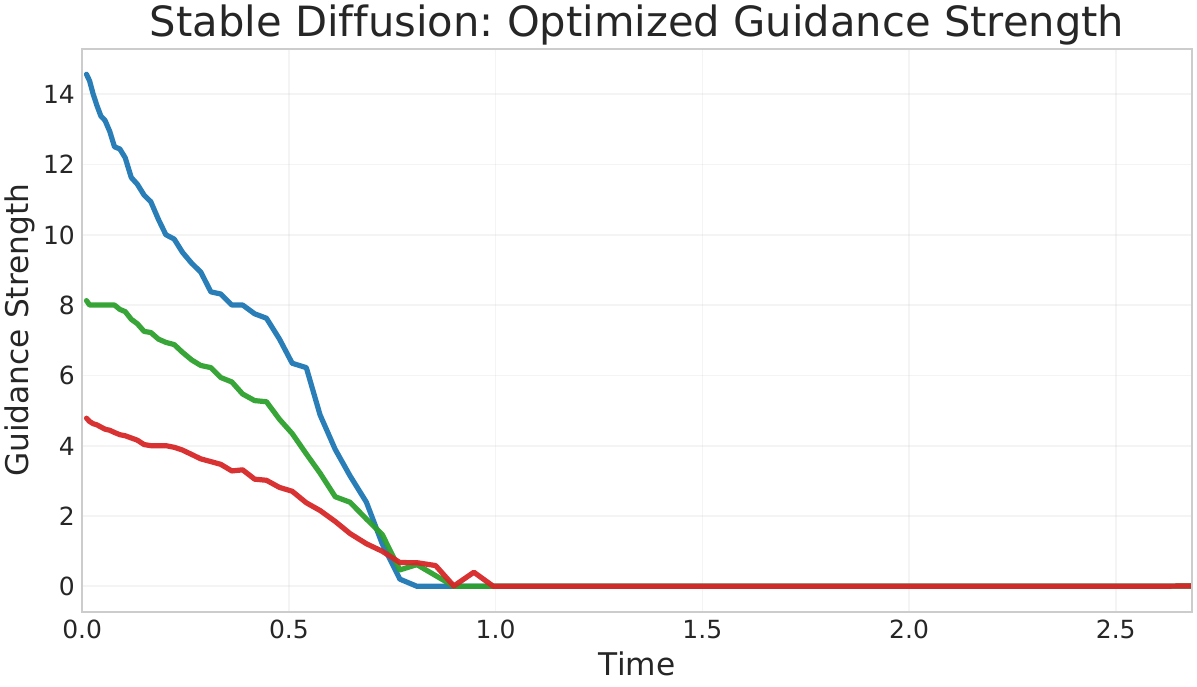}%
    \hspace{0.01\textwidth}%
    \includegraphics[width=0.41\textwidth,valign=t]{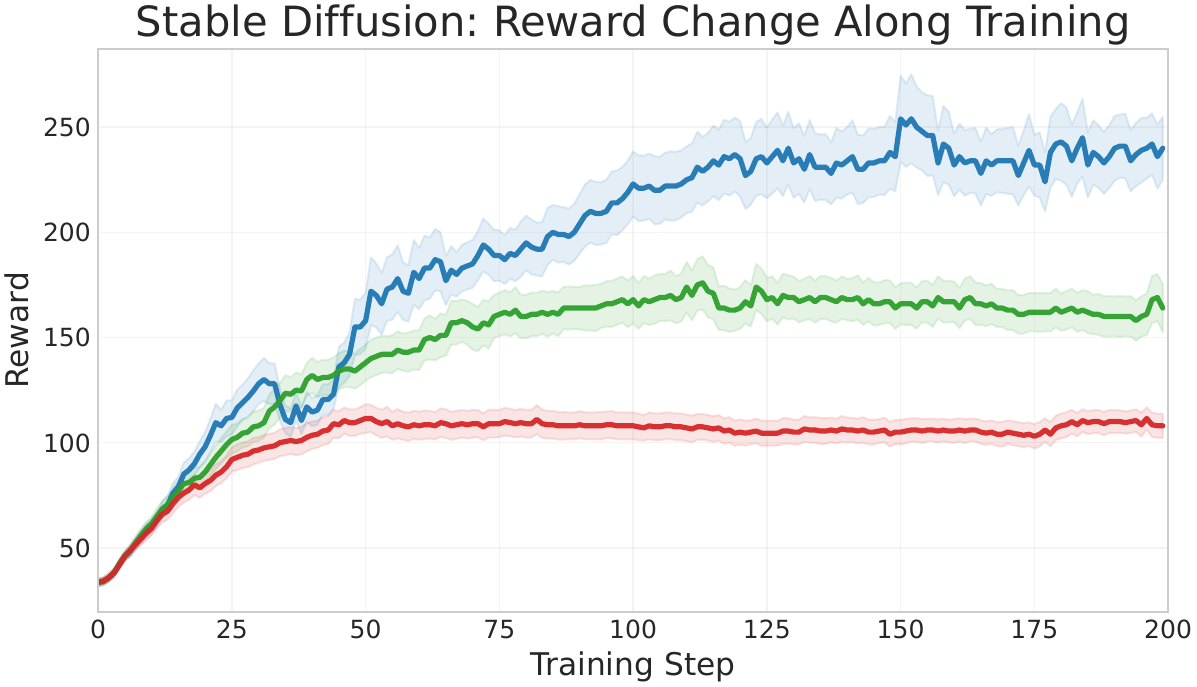}%
    \hspace{0.01\textwidth}%
    \raisebox{-0.3cm}{\includegraphics[width=0.15\textwidth,valign=t]{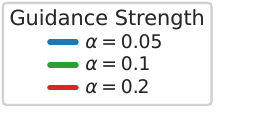}}%
    \caption{Stable Diffusion. Left: Guidance learned at the end of the training. Right: Change of the reward function $R(w_\theta)$ during the training with $\pm 1$ standard deviation.}
    \label{fig:sd-plots}
\end{figure*}

\paragraph{HJB simulation.}

As a first example, we consider a mixture $p(x)$ of four Isotropic Gaussians with variance $0.5$. One is centered at the origin, and the other three are centered on an equilateral triangle around the origin. The distribution $p(x|c)$ is the Gaussian at the origin. To obtain the optimal guidance strength in this case, we solve the Hamilton-Jacobi-Bellman PDE directly. We visualize the optimal guidance weights for different times in Figure \ref{fig:2x3_with_colorbar}. We would like to highlight that the optimal schedule depends on both time $t$ and space $x$, and is not a constant.

\paragraph{EDM2}

We evaluated our guidance optimization framework on the EDM2 model~\citep{Karras2024edm2} (CC BY-NC-SA 4.0), which was pre-trained on the ImageNet dataset~\citep{imagenet_cvpr09}. Specifically, we conducted experiments using a size-M model at a $512 \times 512$ resolution and a size-S model at $64 \times 64$. We fine-tuned the time-dependent guidance $w_\theta(t)$ by maximizing the reward~\eqref{eqn:objective} via the Adjoint Method (detailed in Section~\ref{sec:adjoint method}) with regularization parameters $\alpha \in \{2.5, 5.0, 10.0\}$. The guidance scale was optimized for 25 epochs (20 epochs for the $512 \times 512$ model) using 10 gradient steps per epoch, processing all 1,000 labels once per epoch. Additional implementation details are provided in Appendix~\ref{app:add-details}.

Table~\ref{edm2-table} compares the Fréchet Inception Distance (FID) and external classifier, InceptionV3 \citep{szegedy2016rethinking}, confidence between our trained guidance schedule and a constant $\alpha^{-1}$ baseline. This baseline represents the naive optimal strategy $w = \arg\max ( 1+2w-\alpha w^2)$ derived from Equation~\eqref{eqn:objective}. As shown, the trained schedule yields slightly higher classifier confidence at the cost of a marginally higher FID. Figure~\ref{fig:edm2-plots} illustrates the reward curve alongside the learned time-dependent guidance. The optimized guidance follows a bell-shaped trajectory, starting near $\alpha^{-1}$ at $t=0$ and decaying close to zero for $t>4$.
\paragraph{Stable Diffusion}
We extended our evaluation by training time-dependent guidance for Stable Diffusion 1.5~\citep{rombach2022high} (CreativeML-OpenRail-M license) on both the ImageNet and COCO-captions~\citep{chen2015microsoft} (CC BY 4.0) datasets. The training procedure mirrored the EDM2 setup, with full details available in Appendix~\ref{app:add-details}.

For both datasets, the optimized guidance profiles differ considerably, as depicted in Figures~\ref{fig:sd-plots} and~\ref{fig: training}. In all configurations, the guidance strength begins near $\alpha^{-1}$ at $t=0$, monotonically decreases to $0$ over the interval $[0,1]$, and remains at $0$ for $t>1$. A plausible explanation for this behavior is the difference in conditioning mechanisms: while EDM2 relies on discrete, well-separated class labels, Stable Diffusion utilizes prompts represented as continuous embeddings in a high-dimensional space. In the first case, applying guidance may cease to provide further benefits once the conditional class is established earlier in the process.
\section{Limitations and Future Work}
A natural progression of this work is to train a comprehensive model where guidance dynamically depends on all variables: time, prompt, and the current state. Currently, the primary obstacle to this approach is the observed instability in the gradients of the score network. We leave this issue for future research.

Additionally, a limitation of our proposed guidance optimization method is its assumption of perfect score estimation by the diffusion model. In practice, this assumption is violated, which may inadvertently compound the errors of the estimator.
\begin{table*}[t]
\label{table:edm2}
\caption{Trained guidance perfomance for EDM2 model. Compared to constant $\alpha^{-1}$ guidance, the corresponding trained one ensures slightly higher Classifier~(Inception V3) confidence on the cost of slightly higher FID.} \label{edm2-table}
\begin{center}
\begin{tabular}{cccc}
\textbf{Model Size} & $\alpha$ &$(\downarrow)$ \textbf{FID} & $(\uparrow)$ \textbf{Classifier Confidence} \\
\hline \\
S & 10  &  4.86 (4.14)  & 0.821 (0.817) \\
S & 5   &  7.16 (5.67)  & 0.836 (0.831) \\
S & 2.5 &  10.01 (8.45) & 0.842 (0.841) \\
M & 10  & 4.46 (3.99)  & 0.844 (0.843) \\
M & 5   &  6.62 (5.52)  & 0.835 (0.829) \\
M & 2.5 & 9.13 (8.26)  & 0.844 (0.843) \\
\end{tabular}
\end{center}
\end{table*}
\section{Conclusion}
In this work, we address critical theoretical gaps in the understanding of guidance within diffusion models. Our analysis rigorously establishes that guidance not only enhances alignment with conditioning signals but also crucially preserves the support of the conditional data manifold. Leveraging these foundational insights, we formulate the scheduling of guidance strength as a stochastic optimal control problem and propose practical optimization algorithms based on adjoint state method. This approach facilitates the development of dynamic, sample- and time-dependent guidance policies that can be efficiently trained and deployed. Overall, our work provides a robust theoretical and algorithmic foundation for adaptive guidance, opening new directions for more controlled and effective generation in diffusion models.

\subsubsection*{Acknowledgements}

We thank \href{https://modal.com}{Modal} for providing the computational infrastructure used in our experiments.
IA was supported by the Engineering and Physical Sciences Research Council [grant number EP/T517811/1].
PP and QL are supported by the EPSRC CDT in Modern Statistics and Statistical Machine Learning (EP/S023151/1)
GD was supported by the Engineering and Physical Sciences Research Council [grant number EP/Y018273/1].
%The project leading to this work has 
JR received funding from the European Research Council (ERC) under the European Union’s Horizon 2020 research and innovation programme (grant agreement No 834175)

\bibliographystyle{apalike}
\bibliography{references.bib}
\appendix
% Supplementary material: To improve readability, you must use a single-column format for the supplementary material.
\onecolumn
\section{MISSING PROOFS AND CALCULATIONS}
\label{app:proofs}
\subsection{Proof of Lemma 1, Section 3}
\begin{proof}
Consider a forward process
\begin{equation}
    dX_t = - X_tdt + \sqrt{2}dB_t, \quad t \in [0, T].
\end{equation}
Let $p_t$ be the density of the marginal $X_t$ at time $t$, Fokker-Plank equation states
\[
\frac{d}{dt} p_t(x) = \grad \cdot (x p_t) + \Delta p_t = d p_t +  \innerproduct{x}{\grad p_t} + \Delta p_t.
\]
Dividing both parts by $p_t$
\[
\frac{d}{dt} \log p_t(x) = d +  \innerproduct{x}{\grad \log p_t} + \frac{\Delta p_t}{p_t} = d +  \innerproduct{x}{\grad \log p_t} + \Delta \log p_t + \norm{\grad \log p_t}^2
\]
So, substituting $p_t(x)$ and $p_t(x|c)$ we get
\£
\label{eq:d/dtG_t}
\frac{d}{dt}G_t(x) &= \frac{d}{dt}\Bigl[\log p_{T-t}(x|c)-\log p_{T-t}(x)\Bigr] 
\\
&
=
-\innerproduct{x}{\grad G_t(x)} - \Delta G_t(x) + \norm{\grad G_t(x)}^2 - 2\innerproduct{\grad \log p_{T-t}(x|c)}{\grad G_t(x)}. \nonumber
\£

Let $\mu_t^w(x) = x+ 2 \grad \log p_{T-t}(x|c) + 2w\grad G_t(x)$ denote the drift of $Y^w_t$. Applying Ito's lemma, and then substituting~\eqref{eq:d/dtG_t} we finish the proof
\[
dG_t(Y^w_t) &= \square{\frac{d}{dt}G_t(Y^w_t) + \innerproduct{\mu_t^w(Y^w_t)}{\grad G_t(Y^w_t)} + \Delta G_t(Y^w_t)}dt + \sqrt{2}\grad G_t(Y^w_t)\cdot dB_t
\\
&=(1 + 2w_t)\norm{\grad G_t(Y^w_t)}^2dt + \sqrt{2}\grad G_t(Y^w_t)\cdot dB_t.
\]
\end{proof}
\subsection{Proof of Corollary 2, Section 3}
\begin{proof}    
First, we verify that $S_t$ is the stochastic exponential of the martingale $-\sqrt{2}\nabla G_s(Y_s^w) \cdot dB_s$. Recall that 
$$ 
S_t = \exp\left( - G_t + \int_0^t 2w_s \|\nabla G_s\|^2ds\right).$$
Let 
\[
I_t := -G_t + \int_0^t 2w_s\norm{\grad G_s}^2ds,
\]
note that 
\[
d\langle I\rangle_t = d\langle G \rangle_t = 2\norm{\nabla G_t}^2dt. 
\]
By Ito's lemma and Lemma~1 we get
\begin{align*}
 dS_t &= S_t\, dI_t + \frac{1}{2}S_t\, d\langle I\rangle_t = S_t\del{-dG_t + 2w_t\norm{\grad G_t}^2dt} + \frac{1}{2}S_t\cdot 2\norm{\grad G_t}^2dt\\
 &= -S_t\sqrt{2}\grad G_t\cdot dB_t,
 \end{align*}
%\[
%dS_t = S_t\, dI_t + \frac{1}{2}S_t\, d\langle I\rangle_t = -S_t\del{dG_t + 2w_t\norm{\grad G_t}^2dt} + \frac{1}{2}S_t\cdot 2\norm{\grad G_t}^2dt
%= -S_t\sqrt{2}\grad G_t\cdot dB_t.
%\]
implying that $S_t$ is indeed the exponent of $-\sqrt{2}\nabla G_s(Y_s^w) \cdot dB_s$.
It is known that stochastic exponent is a supermartingale~\cite[2.28]{karatzas1991brownian} and Doob's inequality~\citep[II.1.15]{revuz2013continuous} for supermartingales states that
\[
\P\del{\sup_{t\in [0,T]} S_{t} \ge \lambda} \le \lambda^{-1}\E S_0 
\]
Substituting~$S_t$ we get
\[
\P\del{\sup_{t\in [0,T]} \frac{p(c)}{p(c|Y^w_t)}\exp\del{\int_0^t 2w_t\norm{\grad G_s}^2ds} \ge \lambda} \le \lambda^{-1}\E \frac{p(c)}{p(c|Y^w_0)}= \lambda^{-1},
\]
which is equivalent to the first statement of Corollary~2. 

Next, we show the second statement.
In general, stochastic exponential is only a local martingale, however, it is a true martingale under Novikov's condition~\cite[5.12]{karatzas1991brownian},
% \begin{proposition}[Novikov's condition]
%     If $M_t$ is an adapted process on $[0,T]$ and
%     \[
%     \E\exp \del{\frac{1}{2}\langle M \rangle_T} < \infty
%     \]
%     then $S_t$ solving $dS_t = S_tM_tdB_t$ is a true martingale.
% \end{proposition}
which in our case, means that we need to check whether
\[
    \E\exp \del{\int_0^T \norm{\grad G_t(Y^t_w)}^2 dt} < \infty
\]
As we showed in (17) for $t < T$
\[
\norm{\grad G_t(x)}^2 \le \frac{4e^{2(t-T)}}{(1-e^{2(t-T)})^2}R^2 < \infty,
\]
so $S_t$ is a true martingale on $[0,T)$. We note that so far we are initial condition agnostic, so the same argument works in a more general case when the reference process $Y^w_t$ is substituted with $\bar{Y}_t^w = (Y_t^w|Y_{t_0}^w=x)$ satisfying
\begin{equation*}
    \begin{cases}
        d\bar{Y}_t^w = \square{Y^t_w + 2\grad \log p_{T-t}(\bar{Y}^w_t|c) + 2w\grad G_t(\bar{Y}_t^w) }ds + \sqrt{2}dB_t, & t\in (t_0,T)\\
        \bar{Y}_{t_0}^w = x,
    \end{cases}
\end{equation*}
for $0 \le t_0 < T$. So, the process
\[
\exp\del{-G_t(\bar{Y}_t^w) + \int_{t_0}^t 2w_s\norm{\grad G_s(\bar{Y}_t^w)}^2ds} 
\]
is still a true martingale on $[t_0,T)$ and as result, we get the second part stating that for $t \in [t_0, T)$
    \[
       \E\square{p^{-1}_{T-t}(c|Y^w_{t})\exp\del{\int_{t_0}^{t} 2w\norm{\grad G_s}^2ds}\Big| Y^w_t = x} =  p^{-1}_{T-t_0}(c|x).
    \]
\end{proof}
\subsection{Proof of Theorem 3 and Remark 4, Section 4}
% \SupportRecovery*
% \begin{proof}

The SDE~(5) governing process $Y^w$ is only well defined on the semi-interval $[0,T)$, so to prove that $Y^w_T \in \supp p(\cdot|c)$ we need to show two statements: (i) $Y^w$ can be continuously extended on $[0,T]$, i.e. $Y^w_T = \lim_{t\rightarrow T} Y^w_t$ exists a.s. (ii) The constructed $Y^w_T$ belongs to $\supp p(\cdot|c)$. We address both problems by the application of Donsker and Varadhan’s variational formula.

We recall that the marginal of the forward process is represented as 
\[
X_t \stackrel{\cD}{=} e^{-t}X_0 + \sqrt{1-e^{-2t}}Z,
\]
where $Z\sim \cN\del{0,\Id_d}$. Therefore, to be consistent in $X_0$ we consider the process $e^{T-t}Y^w_t$ instead and define $Y^w_T = \lim_{t\rightarrow T} e^{T-t}Y^w_t$. Note that this limit (if it exists) is equal to $ \lim_{t\rightarrow T} Y^w_t$, since $\lim_{t\rightarrow T} e^{T-t} = 1$.

Let $d$ be the dimension of  $X_0\in \R^d$. We use $A^c=\R^d\backslash A$ to denote the complement of set $A\subset \R^d$ and $cA=\curly{ca: a\in A}$ for a scalar $c\in \R$.

Let $\Omega := \supp p_0(\cdot|c)$ denotes the support and $\Omega_\varepsilon:=\curly{x\in \R^d: \inf_{\omega\in \Omega}\norm{x-\omega} \le \varepsilon}$ denotes the $\varepsilon$-neighborhood of $\Omega$. 

We prove in full generality, assuming that there are constants $\infty > C_{\min}, C_{\max} > 0$
\[
-\frac{1}{2} + C_{\min} \le w \le C_{\max}
\]
Then, using the same reasoning as in Section 4,
\[
\KL(Y^w_{T-\varepsilon}\| Y^0_{T-\varepsilon}) \le \KL(Y^w_{[0,T-\varepsilon]}\| Y^0_{[0,T-\varepsilon]}) \le  \max\del{C_{\max}, \frac{C^2_{\max}}{2C_{\min}}} \log p^{-1}(c) := C_{\KL} < \infty.
\]  

%We were unable to show that process $Y_t$ is continuous on the whole interval $[0,T]$, so instead, we define $Y_T$ via accumulation points, specifically as a (any) limit of $e^{T-t_n}Y_{t_n}$ where $t_n = T-1/n^3$ that converges to $T$ as $n\rightarrow \infty$. 

%For those who want to avoid the accumulation point argument, we also prove a slightly weaker version of the statement stating that for any $\varepsilon > 0$
First we prove the following lemma
\begin{lemma} Under the same conditions as in~Theorem~3, for any $\varepsilon > 0$
\£
\label{eq:week_theorem_3}
   \lim_{t\rightarrow T}\rP\del{e^{T-t}Y^w_t \in \Omega_\varepsilon} = 1.
\£
\end{lemma}

%This will demonstrate the method and is a key for step (ii).
%Arguing by contradiction, it is enough to show $\rP(Y^{w}_t\in  e^{t-T}\Omega_\varepsilon) \rightarrow 1$ for any $\varepsilon >0$. %where $e^{t-T}\Omega^c_\varepsilon = \R^d\backslash \curly{e^{t-T}\Omega_\varepsilon}$.
\begin{proof}    
We recall \cite[Corollary 4.15]{boucheron2013} the duality formula for $\KL$ stating that for two probability distributions $P$ and $Q$
\[
\KL(Q\|P) = \sup_{Z: \E_P e^Z < \infty} \E_Q Z - \log \E_P e^Z.
\]
Substituting $Z = \alpha \ind_{e^{t-T}\Omega^c_\varepsilon}$ for some $\alpha > 0$ and as $Q$ and  $P$ laws of $Y^w_t$ and $Y^0_t$ we get 
    \[
    \KL\del{Y^{w}_t\| Y^0_t} + \log [\rP\del{Y^0_t \in e^{t-T}\Omega_\varepsilon} + e^{\alpha}\rP\del{Y^0_t \not\in e^{t-T}\Omega_\varepsilon}] \ge \rP\del{Y^w_t \not\in e^{t-T}\Omega_\varepsilon}
    \]
So, for any $\alpha > 0$
    \£
    \label{eq:variation_bound}
    \frac{C_{\KL} + \log [1 + e^{\alpha}\rP\del{Y^0_t \not\in e^{t-T}\Omega_\varepsilon}]}{\alpha}
    &\ge
    \frac{\KL\del{Y^{w}_t\| Y^0_t} + \log [\rP\del{Y^0_t \in e^{t-T}\Omega_\varepsilon} + e^{\alpha}\rP\del{Y^0_t \not\in e^{t-T}\Omega_\varepsilon}]}{\alpha} \nonumber
    \\
    &\ge \rP\del{Y^w_t \not\in e^{t-T}\Omega_\varepsilon}.
    \£
    Finally, we note that $Y^0_t$ is a true backward process, so $Y^0_t \stackrel{\cD}{=} e^{t-T}Y^0_T + \sqrt{1-e^{2(t-T)}}Z$, where $Z\sim \cN\del{0,\Id_D}$ and $Y^0_T\sim p(\cdot|c)$, therefore
    \[
    \rP\del{Y^0_t \not\in e^{t-T}\Omega_\varepsilon} = \rP\del{e^{t-T}Y^0_T + \sqrt{1-e^{2(t-T)}} Z \not\in e^{t-T}\Omega_\varepsilon} \le \rP\del{\sqrt{1-e^{2(t-T)}} \norm{Z} \ge e^{t-T}\varepsilon} \rightarrow_{t\rightarrow T} 0.
    \]
    So, taking the limit $t\rightarrow T$ we get that for all $\alpha > 0$
    \[
    \frac{C_{\KL}}{\alpha} \ge \lim_{t\rightarrow T}\rP\del{Y^w_t \not\in e^{t-T}\Omega_\varepsilon}
    \]
By taking $\alpha \rightarrow \infty$
%, since $\diam \Omega \le R$ 
we get \eqref{eq:week_theorem_3}.
\end{proof}

\begin{lemma}
    Under the same conditions as in~Theorem 3, the process 
    $e^{T-t}Y^w_t$ a.s. admits a continuous extension on $[0,T]$.
\end{lemma}
\begin{proof}    
We show this by checking Cauchy type condition. We fix $T_i = T-2^{-i}$ and prove that a.s.
\£
\label{eq:telescopic}
\sum_{i=1}^\infty \sup_{s,t\in [T_i,T_{i+1}]}\norm{e^{T-s}Y^w_s -e^{T-t}Y^w_t} < \infty.
\£
Then $Y^w_T$, for example, can be constructed as a telescopic limit
\[
Y^w_T :=e^{1/2}Y^w_{T-1/2} + \sum_{i=1}^\infty (e^{T-T_{i+1}}Y^w_{T_{i+1}}-e^{T-T_{i}}Y^w_{T_{i}}),
\]
and~\eqref{eq:telescopic} will guarantee that the sum converges and it does not depend on the choice of $\curly{T_i}$.

The forward OU process $X_t$ can be represented as a rescaled Wiener process, more precisely
\[
X_{t} = e^{-t}\del{X_0 + W_{e^{2t}-1}}.
\]
So, for $0 \le A < B < \infty$
\[
\sup_{s,t \in [A,B]} \norm{e^tX_{t}-e^sX_{s}} = \sup_{s,t \in [A,B]} \norm{W_{e^{2t}-1}-W_{e^{2s}-1}} \le 2\sup_{t}\norm{ W_{e^{2t}-1}- W_{e^{2A}-1}},
\]
and by the reflection principle for $\delta  < 1/4$%, since $\log(2d\delta^{-1}) < (\log 8d)\log\delta^{-1}$
\[
\rP\del{\sup_{s,t \in [A,B]} \norm{e^tX_{t}-e^sX_{s}} \ge 2\sqrt{2d(\log 8d)\del{e^{2B}-e^{2A}}\log \delta^{-1}}} \\
\le \rP\del{2\sup_{t}\norm{ W_{e^{2t}-1}- W_{e^{2A}-1}} \ge 2\sqrt{2d\del{(e^{2B}-1)-(e^{2A}-1)}\log (2d\delta^{-1})}} \le \delta,
\]
where we additionally used  $\log(2d\delta^{-1}) < (\log 8d)\log\delta^{-1}$.
So, for $A, B \le 1/2$
\[
\rP\del{\sup_{s,t \in [A,B]} \norm{e^tX_{t}-e^sX_{s}} \ge 8\sqrt{d(\log 8d)\del{B-A}\log \delta^{-1}}} \le \delta.
\]
% Since the OU process is self-similar
% \[
% X_{t+\gamma} = e^{-\gamma}X_{t} + \sqrt{1-e^{-2\gamma}}Z,
% \]
% so for $t < \log 2$
% \[
% \rP\del{\norm{e^{t+\gamma}X_{t+\gamma} - e^{t}X_{t}}\ge 4\sqrt{d\gamma \log \delta^{-1}}}\le \delta
% \]
% or equivalently
% \[
% \rP\del{\norm{e^{t_i-T}Y_{t_i} -e^{t_{i+1}-T}Y_{t_{i+1}}} \ge 4\sqrt{d(t_i-t_{i+1}) \log \delta_n^{-1}}} \le \delta_n
% \]
As result, by~\eqref{eq:variation_bound}
\£
    \frac{C_{\KL} + \log [1 + e^{\alpha_n}\delta_n]}{\alpha_n}
    \ge \rP\del{\sup_{s,t\in [T_i,T_{i+1}]}\norm{e^{T-s}Y^w_{s} -e^{T-t}Y^w_{t}} \ge 8\sqrt{d(\log 8d)(T_i-T_{i+1}) \log \delta_n^{-1}}}.
\£
We take $\alpha_n = \gamma^{-1} n^2$ and $\delta_n = \exp\del{-\gamma^{-1}n^2}$, so
\[
    \gamma\frac{C_{\KL} + \log 2}{n^2}
    \ge 
    \rP\del{\sup_{s,t\in [T_i,T_{i+1}]}\norm{e^{T-s}Y^w_{s} -e^{T-t}Y^w_{t}} \ge 8\sqrt{d(\log 8d)(T_i-T_{i+1}) \gamma^{-1} n^2}}
\]
Summing over all $T_i$, since $T_{i+1}-T_i = 2^{-(i+1)}$ and
\[
\sum \frac{n}{2^{n/2}} = 4+3\sqrt{2} < 10, \quad \sum \frac{1}{n^2} = \frac{\pi^2}{6},
\]
we get the union bound 
\[
\gamma (C_{\KL} + \log 2)\frac{\pi^2}{6} > \rP\del{\sum_{i}\sup_{s,t\in [T_i,T_{i+1}]}\norm{e^{T-s}Y^w_{s} -e^{T-t}Y^w_{t}} \ge 80\sqrt{\gamma^{-1}} \sqrt{d\log 8d}}.
\]
By taking $\gamma\rightarrow 0$ we show \eqref{eq:telescopic}.
%\end{proof}
\end{proof}

Since almost sure convergence implies weak convergence and $\Omega_\varepsilon$ is closed
\[
1 = \lim_{t\rightarrow T}\rP\del{Y^w_t \in \Omega_\varepsilon} = \limsup_{t\rightarrow T}\rP\del{Y^w_t \in \Omega_\varepsilon} 
\le
\rP\del{Y^w_T \in \Omega_\varepsilon}.
\]
We finish the proof of~Theorem 3 and Remark 4 by taking $\varepsilon \rightarrow 0$.

\subsection{Omitted Calculations in Section~4}
In this Section we present calculations ommited in Section 4.

We recall that the sensitivity process $Z_t(\tau)$ is defined as
\[
Z_t(\tau) = \frac{\partial Y^w_t}{\partial w_\tau}.
\]
Applying the chain rule to the reward $R(w)$ we obtain
\£
\label{eq: sensitivity of R}
    \frac{\partial R(w)}{\partial w_\tau}
    &= \frac{\partial}{\partial w_\tau}\E\int_0^T (1+2w_t-\alpha w_t^2)\norm{\grad G_t(Y_t^w)}^2dt \nonumber
    \\
    &=\E\int_0^T 2\cdot\delta(\tau-t)(1 - \alpha w_t)\norm{\grad G_t(Y_t^w)}^2 + 2(1+2w_t-\alpha w_t^2)\grad^T G_t(Y_t^w)\grad^2 G_t(Y_t^w)\frac{\partial Y^w_t}{\partial w_\tau} dt \nonumber
    \\
    &=2\E (1 - \alpha w_\tau)\norm{\grad G_t(Y_t^w)}^2 + \expectation\int_\tau^T{\color{blue}{2(1+2w_t-\alpha w_t^2)\grad^T G_t(Y_t^w)\grad^2 G_t(Y_t^w)}}Z_t(\tau) dt 
\£

Next, using the chain rule, we write the differential equation on the sensitivity process
\[
dZ_t(\tau) &= d\frac{\partial Y^w_t}{\partial w_\tau} =  \frac{\partial}{\partial w_\tau}\del{\square{Y^t_w + 2\grad \log p_{T-t}(Y^w_t|c) + 2w_t\grad G_t(Y_t^w) }dt + \sqrt{2}dB_t}
\\
&=\square{Z_t(\tau) + 2\grad^2 \log p_{T-t}(Y^w_t|c)Z_t(\tau) + 2\delta(\tau-t)\grad G_\tau(Y_\tau^w)+ 2w_t\grad^2 G_t(Y_t^w)Z_t(\tau)}dt
\]
implying that it satisfies ODE
\begin{equation}
\begin{cases}
dZ_t(\tau) =\square{\color{red}\Id + 2\grad^2 \log p_{T-t}(Y^w_t|c)+ 2w_t\grad^2 G_t(Y_t^w)}Z_t(\tau)dt,
\\
Z_\tau = 2\grad G_\tau(Y_\tau^w).
\end{cases}
\end{equation}
We then consider an adjoint backward process given by 
\begin{equation}    
\begin{cases}    
d\lambda_t = -\square{\color{red} \underbrace{\Id + 2\grad^2 \log p_{T-t}(Y^w_t|c)+ 2w_t\grad^2 G_t(Y_t^w)}_{:=A_t}}^T\lambda_t dt -{\color{blue}\underbrace{2(1+2w_t-\alpha w_t^2)\grad^T G_t(Y_t^w)\grad^2 G_t(Y_t^w)}_{:=B_t}}dt
\\
\lambda_T = 0
\end{cases}
\end{equation}
Then, integrating by parts, and noting that $\innerproduct{x}{Wy} = \innerproduct{W^Tx}{y}$, we see that 
\[
\langle \lambda_\tau, 2\grad G_\tau(Y_\tau^w) \rangle &= \innerproduct{\lambda_\tau}{Z_\tau} = \innerproduct{\lambda_\tau}{Z_\tau} - \innerproduct{\lambda_T}{Z_T}= -\int_\tau^T \del{\innerproduct{\frac{d}{dt}\lambda_t}{Z_t} + \innerproduct{\lambda_t}{\frac{d}{dt}Z_t}}dt 
\\
&= -\int_\tau^T \del{-\innerproduct{{\color{red} A_t}^T\lambda_t - {\color{blue} B^T_t}}{Z_t} + \innerproduct{\lambda_t}{{\color{red} A_t}Z_t}}dt = \int_\tau^T {\color{blue}{2(1+2w_t-\alpha w_t^2)\grad^T G_t(Y_t^w)\grad^2 G_t(Y_t^w)}} Z_t dt
\]
that simplifies \eqref{eq: sensitivity of R} and results in
\[
\frac{\partial R(w)}{\partial w_\tau}=2\E (1 - \alpha w_\tau)\norm{\grad G_t(Y_t^w)}^2 + 2\langle \lambda_\tau, \grad G_\tau(Y_\tau^w) \rangle
\]

\section{ADDITIONAL EXPERIMENTS \& DETAILS}
\label{app:add-details}
\subsection{Additional Experimental Details}
\label{app:experimental-details}

\paragraph{Model specifications.}
We consider two pretrained diffusion model families: Stable Diffusion v1.5 and EDM2. Stable Diffusion is used in its standard latent-diffusion form, with denoising performed in a $4\times 64\times 64$ latent space corresponding to $512\times 512$ images, using the model's pretrained VAE encoder and decoder without modification. EDM2 is evaluated on pretrained ImageNet checkpoints; the $512\times 512$ models are likewise represented in a $4\times 64\times 64$ latent space, whereas the $64\times 64$ models operate directly in pixel space and therefore do not use a VAE. In all cases, the underlying generative model is frozen and only the scalar guidance schedule is trained.

\subparagraph{Data}
For Stable Diffusion experiments on the ImageNet dataset, conditioning prompts are constructed from the class names using the template ``a photo of \{class name\}'', with underscores replaced by spaces. Training was performed on all $1000$ ImageNet classes, with $20$ holdout prompts fixed to track performance during training. For EDM2, conditioning is performed directly using ImageNet class labels rather than text prompts.

\begin{figure*}[htbp]
    \centering
    % Left grid of 2 rows × 3 columns
    \includegraphics[width=0.18\linewidth]{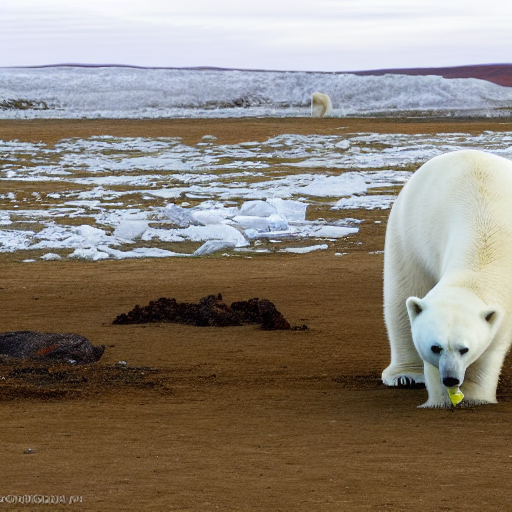}
    \includegraphics[width=0.18\linewidth]{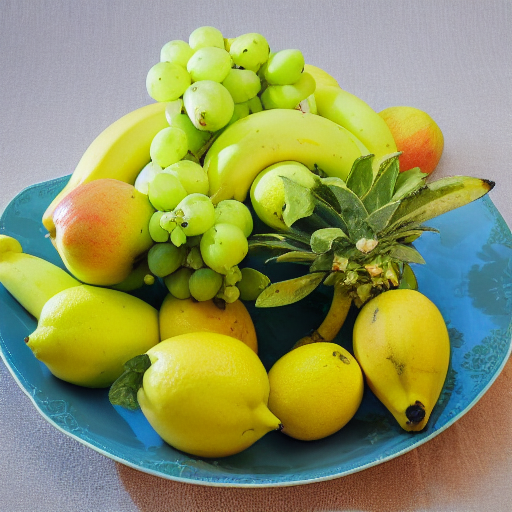}
    \includegraphics[width=0.18\linewidth]{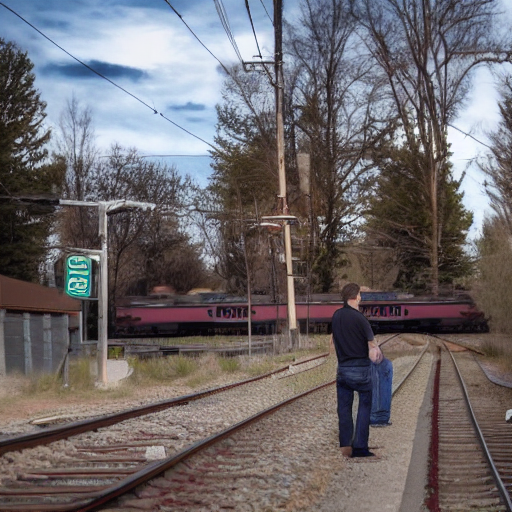}
    \includegraphics[width=0.18\linewidth]{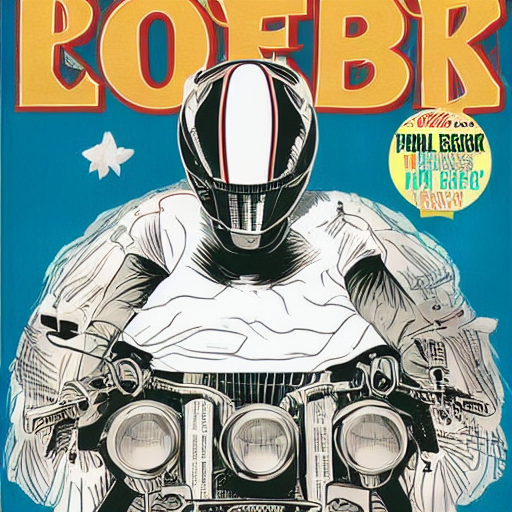}
    \includegraphics[width=0.18\linewidth]{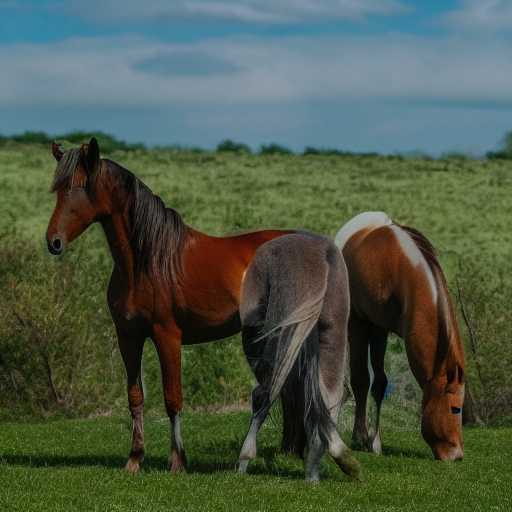}
    \caption{Images sampled using Stable Diffusion with guidance trained on the COCO dataset. Captions are taken from the holdout set. All images sampled using 100 points. The list of corresponding prompts from left to right: 
    "The polar bear is looking at his latest treat.", 
    "A plate of fruit such as bananas, lemons, apples, and oranges.",
    "A man stands between two railroad tracks as a train approaches",
    "Open book with a magazine cover of motorcycle",
    "a couple of horses grazing on some green grass".}
    \label{fig: samples sd}
\end{figure*}

\paragraph{Sampling and guidance schedule parameterization.}
In all experiments, reverse diffusion uses $64$ steps during training and $32$ steps during evaluation. Sampling uses Heun discretization, with an early-stopping cutoff of $t=0.01$ during training and $2\times 10^{-6}$ during evaluation. For Stable Diffusion, the full diffusion process is discretized into $1000$ timesteps, which index the noise schedule of the model; the sampled timesteps are chosen by taking evenly spaced indices from this ordered sequence after applying the early-stopping cutoff. For EDM2, we first construct a dense timestep sequence from the VE noise schedule used in the EDM2 setup, with $\sigma_{\max}=80$, $\sigma_{\min}=0.002$, $\rho=7$, and $1025$ base noise levels, convert it to VP time by SNR matching, and then select the $64$ or $32$ sampling steps by the same evenly spaced indexing procedure.

\paragraph{Guidance parameterization.}
The guidance schedule is parameterized by one scalar parameter for each of the $64$ training timesteps, so the learned schedule consists of $64$ parameters in total, each associated with a specific point on the training-time sampling grid. To ensure that the guidance scale remains positive, each parameter is passed through a softplus transformation before being used at runtime. During evaluation, the sampling grid contains $32$ timesteps rather than $64$; when an evaluation timestep does not exactly coincide with a training-grid point, it is mapped to the nearest training timestep, and the corresponding learned scalar is used.

\paragraph{Training details.}
For Stable Diffusion, the default training configuration uses $20$ epochs, batch size $20$, micro-batch size $5$, Adam, learning rate $0.1$, and gradient clipping at $10.0$. For EDM2, the default training configuration uses $20$ epochs for $512\times512$ model and $25$ epochs for $64\times 64$ model, batch size $100$, micro-batch size $25$, SGD, learning rate $0.01$, and gradient clipping at $10.0$. In both cases, antithetic sampling is enabled during training. Mixed precision is enabled automatically on CUDA devices, using \texttt{bfloat16} when available and otherwise \texttt{float16}.

\paragraph{Evaluation details.}
Evaluation follows the same large-scale class-conditional protocol used in the EDM2 setup: $50{,}000$ generated samples over $1000$ ImageNet classes, with $50$ images per class and random seed $0$. We report FID, FD-DINOv2, Inception accuracy, and Inception confidence. For each setting, the learned timestep-dependent schedule corresponding to regularization parameter $\alpha$ is compared against a constant-guidance baseline with guidance strength $\alpha^{-1}$.

\begin{figure*}[htbp]
    \centering
    % Left grid of 2 rows × 3 columns
    \includegraphics[width=0.33\linewidth]{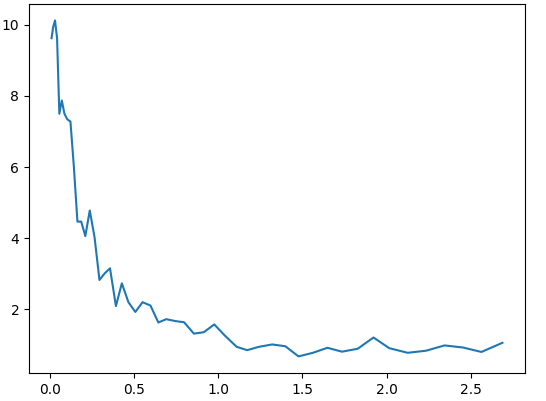}
    \includegraphics[width=0.6\linewidth]{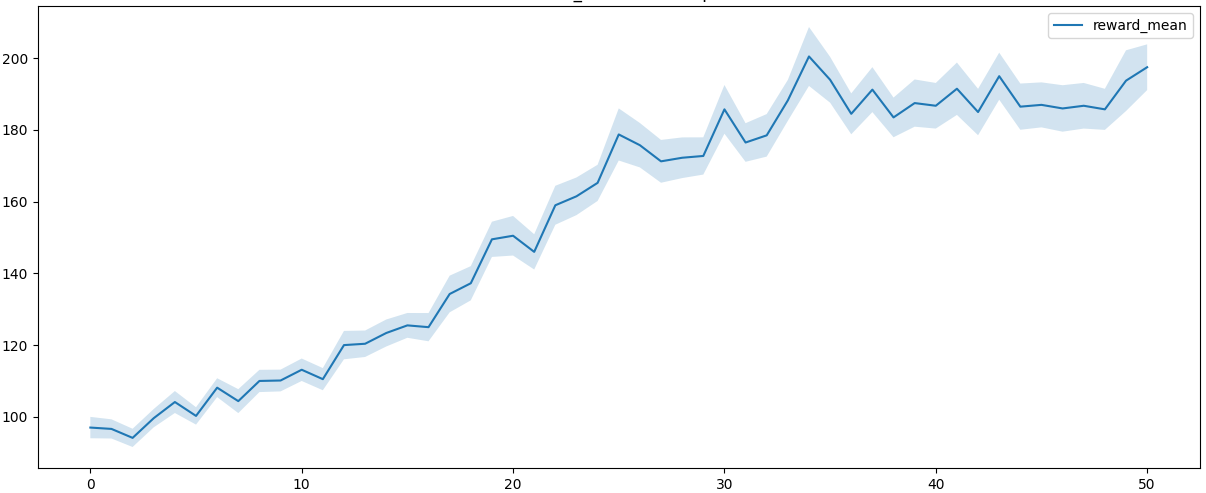}
    \caption{Guidance optimization for COCO dataset. Left: Guidance learned in 50 training steps. Right: Change of the mean reward $R(w_\theta)$ over training steps together with $\pm1$ standard deviation.}
    \label{fig: training}
\end{figure*}

\subsection{COCO experiment}
We additionally conducted a small-scale experiment on the COCO-captions~\citep{chen2015microsoft}(CC BY 4.0) dataset. We trained time-dependent guidance for Stable Diffusion~1.5~\citep{rombach2022high}(CreativeML-OpenRail-M license). We set the regularization parameter to be equal to $\alpha=0.1$ and used the first $576$ prompts from the COCO-captions dataset to fine-tune guidance, while the next $32$ prompts were used as a holdout set to evaluate the performance. 

We trained the time-dependent guidance strength $w_t(\theta)$ for 50 gradient steps on the time interval $[0.01, 2.68]$, discretized into 50 points, by parameterizing $\theta$ with a small MLP with $\exp$ applied to the output. The batch size was chosen to be equal to 72, and the number of trajectories was 4 per prompt.

To improve stability, we have used an early stopping time $\delta=0.01$ and applied quantile clipping while computing $dR/dw_t$.  Finally, following~\cite{domingoenrich2025adjointmatchingfinetuningflow}, we dropped the $\grad^2G_t \lambda_t$ term during simulation of the adjoint process.

\cref{fig: training} contains the learned guidance strength and the reward curve, and \cref{fig: samples sd} are images sampled from the trained guided diffusion.

\newpage
\section{PSEUDOCODE}
\label{sec:discrete-soc-pseudocode}

\begin{algorithm}[htbp]
\caption{Vector Jacobian Product (VJP)}
\label{alg:vjp}
%\small
\begin{algorithmic}[1]
\Require Differentiable vector field $f$, point $y$, vector $v$
\State $\phi(y)\gets \langle f(y), v\rangle$
\State \Return $\nabla_y \phi(y)=\bigl(\nabla_y f(y)\bigr)^\top v$
\end{algorithmic}
\end{algorithm}
\begin{algorithm}[htbp]
\caption{Discrete SOC Optimization via Adjoint Method}
\label{alg:discrete-soc}
\begin{algorithmic}[1]
\Require Parameters $\theta$, time grid $0=t_0<t_1<\cdots<t_N=T$, learning rate $\eta$
\Repeat
    \State Sample a minibatch of conditions $\{c\}$ and initialize the reverse process.
    \State \textbf{Forward Pass:} Simulate the controlled trajectory $\{Y_k\}_{k=0}^N$ conditioned on $c$ on the grid using $w_\theta$.
    \State Evaluate and cache discrete guidance values $w_k \gets w_\theta(t_k,Y_k,c)$ and score terms $\nabla G_k \gets \nabla G_{t_k}(Y_k)$.
    \State \textbf{Backward Pass:} Initialize the terminal adjoint state $\lambda_N \gets 0$.
    \For{$k=N-1, N-2, \dots, 0$}
        \State $\Delta t_k \gets t_{k+1}-t_k$
        \State Define the drift function $f_k(y) \gets y + 2\nabla \log p_{T-t_k}(y \mid c) + 2w_k \nabla G_{t_k}(y)$
        \State Compute $A_k^\top \lambda_{k+1}$ using \cref{alg:vjp}:
        \State \quad $A_k^\top \lambda_{k+1} \gets \operatorname{VJP}(f_k, Y_k, \lambda_{k+1})$
        \State Compute the source term $B_k$ of the adjoint ODE using AutoDiff:
        \State \quad $B_k \gets (1+2w_k-\alpha w_k^2)\grad\norm{\nabla G_{t_k}(Y_{t_k}}^2$
        \State Update the adjoint state via Euler integration:
        \State \quad $\lambda_k \gets \lambda_{k+1} + \Delta t_k (B_k + A_k^\top \lambda_{k+1})$
    \EndFor
    \State \textbf{Gradient Computation:}
    \For{$k=0, 1, \dots, N$}
        \State Compute the gradient of the objective w.r.t. the discrete guidance $w_k$:
        \State \quad $g_k \gets 2(1-\alpha w_k)\|\nabla G_k\|^2 + 2\langle \lambda_k, \nabla G_k \rangle$
    \EndFor
    \State Compute the full parameter gradient via the chain rule:
    \State \quad $\widehat{\nabla_\theta R} \gets \sum_{k=0}^{N} g_k \frac{\partial w_\theta(t_k,Y_k,c)}{\partial \theta} \Delta t_k$
    \State Perform a gradient ascent step:
    \State \quad $\theta \gets \theta + \eta \widehat{\nabla_\theta R}$
\Until{convergence}
\end{algorithmic}
\end{algorithm}
\end{document}